\pdfoutput=1

\documentclass[11pt]{article}

\usepackage[]{EMNLP2022}

\usepackage{times}
\usepackage{latexsym}

\usepackage[T1]{fontenc}

\usepackage[utf8]{inputenc}

\usepackage{microtype}
\usepackage{inconsolata}

\usepackage{url}
\usepackage{color}
\usepackage{makecell}
\usepackage{bbm}
\usepackage{amsmath}
\usepackage{arydshln}
\usepackage{subcaption}
\usepackage{caption}
\usepackage{multirow}
\usepackage{graphicx}
\usepackage{bm}
\usepackage{amsthm}
\newtheorem{theorem}{Theorem}[section]

\title{The Importance of Being Parameters: \\An Intra-Distillation Method for Serious Gains}

\author{Haoran Xu, Philipp Koehn,  Kenton Murray\\[1em]
Johns Hopkins University\\
\texttt{\{hxu64,phi,kenton\}@jhu.edu}\\[1em]
}
\date{}
\begin{document}
\maketitle
\begin{abstract}
Recent model pruning methods have demonstrated the ability to remove redundant parameters without sacrificing model performance. Common methods remove redundant parameters according to the parameter sensitivity, a gradient-based measure reflecting the contribution of the parameters. In this paper, however, we argue that redundant parameters can be trained to make beneficial contributions. We first highlight the large sensitivity (contribution) gap among high-sensitivity and low-sensitivity parameters and show that the model generalization performance can be significantly improved after balancing the contribution of all parameters. Our goal is to balance the sensitivity of all parameters and encourage all of them to contribute equally. We propose a general task-agnostic method, namely \textbf{intra-distillation}, appended to the regular training loss to balance parameter sensitivity. Moreover, we also design a novel adaptive learning method to control the strength of intra-distillation loss for faster convergence. Our experiments show the strong effectiveness of our methods on machine translation, natural language understanding, and zero-shot cross-lingual transfer across up to 48 languages\footnote{Code is available at \url{https://github.com/fe1ixxu/Intra-Distillation}.}, e.g., a gain of 3.54 BLEU on average across 8 language pairs from the IWSLT'14 dataset.
\end{abstract}

\section{Introduction}
\begin{figure}[ht]
    \centering
    \resizebox{0.7\linewidth}{!}{
    \includegraphics[width=7.5cm]{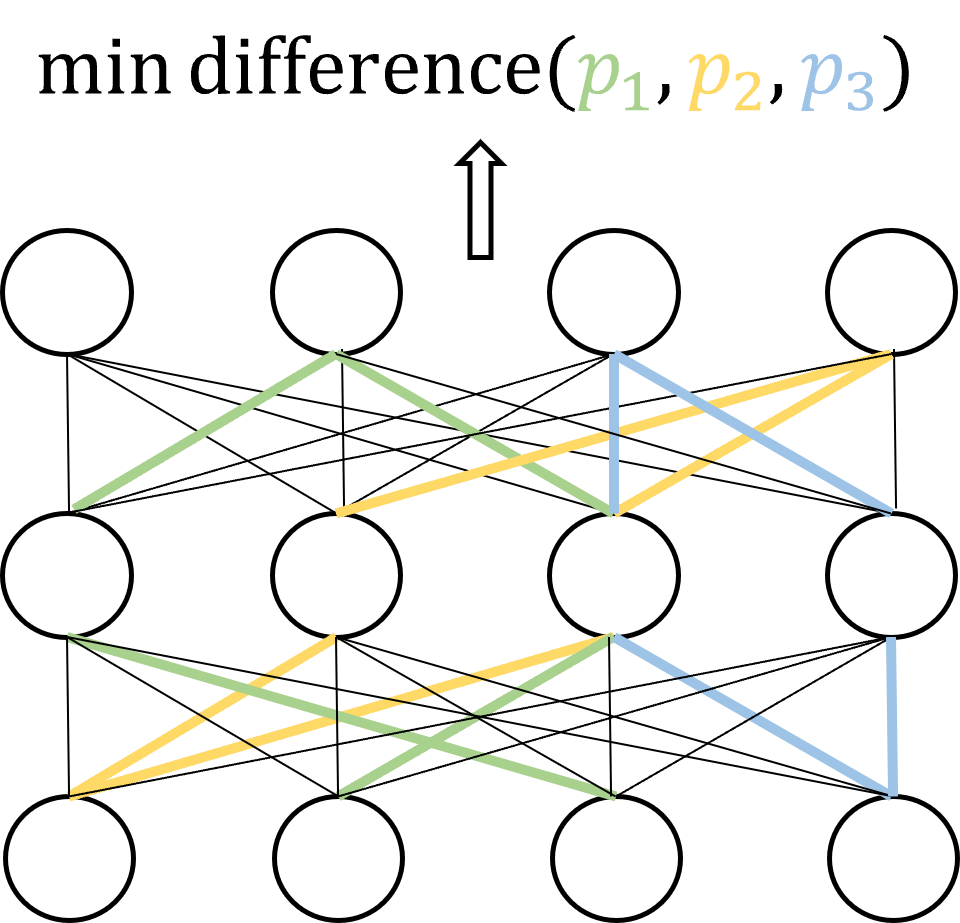}}
    \caption{Illustration of the intra-distillation process. We pass the model $K$ ($K=3$) times and obtain three outputs $p_1, p_2, p_3$. Each time we disable different subsets of parameters (illustrated by different colors). Minimizing the difference of the $K$ outputs approximates minimizing the contribution difference of these disabled subsets of parameters, which can significantly improve model generalization performance. }
    \label{fig:intro}
\end{figure}
Exploring efficient parameter use in neural networks is critical for improving computational quality and decreasing storage requirements \citep{han2015learning,li2016pruning}. The lottery ticket hypothesis \citep{frankle2018lottery} suggests that a small subset of parameters, namely `winning tickets', can reach similar performance compared to a dense model through iterative retraining. Recent techniques for pruning models \citep{lubana2020gradient,sanh2020movement,xiao2019autoprune,molchanov2016pruning} have shown to be successful in reducing redundant parameters of trained networks by over 80\% without obvious loss of the model quality. Despite the success of winning tickets, it usually does not offer better performance and is actually computationally expensive to obtain (needing iterative pruning and retraining). Moreover, unstructured pruning barely accelerates inference computation because the device still computes the dense tensors, with the only difference being that they are filled with zeros. 

The main motivation behind pruning is the existence of redundant parameters which basically have no contribution to the model. Taking an opposite approach from pruning redundant parameters, we encourage all parameters to contribute. \citet{Qiao_2019_CVPR,liang2021no} recently showed that redundant parameters are insufficiently trained, and can actually contribute more when they are trained properly. Following this line, we also argue that there is a large room for improving model generalization by making redundant parameters contribute instead of discarding them. However, our approach differs in that we change the training objective, as opposed to learning rates.

In this paper, we show significant improvement after balancing the contribution of all parameters on various tasks. Our goal is to \textbf{balance the sensitivity of parameters to encourage the equal contribution of each parameter}, where sensitivity is a gradient-based measure reflecting the degree of parameter contribution. Usually, lower-sensitivity parameters are considered redundant. However, in an extreme case of our goal, no parameter is redundant. Thus, we propose \textbf{intra-distillation}, a task-agnostic method aiming to minimize the sensitivity difference among each subset of parameters. Specifically, we obtain $K$ outputs by forward passing the model $K$ times, where we randomly disable a different subset of parameters for each pass (Figure \ref{fig:intro}). We deduce that minimizing the difference of these $K$ outputs approximates minimizing the sensitivity of the disabled parameters. Therefore, in each step of training, we can minimize the sensitivity of $K$ random groups of parameters. We list our main contributions are summarized as follows:
\begin{itemize}
  \itemsep0em 
   \item We introduce a new concept, i.e., the degree of contribution balance, describing how balanced the contribution of all parameters is. This allows us to formally define and measure how parameters can improve task performance. Moreover, we use balanced contribution of parameters to explain the successful `dark knowledge' transfer in knowledge distillation \citep{hinton2015distilling} between students and teachers who use the same architecture (termed self-distillation \citep{furlanello2018born}) (Section \ref{sec:analysis}).
  \item We propose the intra-distillation method with its adaptive strength control, which highly balances the sensitivity (contribution) of model parameters and leads to significantly better generalization performance (Section \ref{sec:method}). 
  \item We conduct wide-ranging experiments on machine translation, natural language understanding, and zero-shot cross-lingual transfer that show intra-distillation outperforms multiple strong baselines by a large margin, e.g., \textbf{3.54} BLEU point gains over the transformer model on average across 8 language pairs from the IWSLT'14 translation task (Section \ref{sec:experiments}).

\end{itemize}

\section{Why Balance the Contribution?}
\label{sec:analysis}
We investigate the contribution difference among parameters based on an important metric,  parameter sensitivity, which has been widely used in pruning under the name ``importance scores'' \citep{ding2019global, molchanov2019importance,lubana2020gradient}. Then, we highlight that  model performance benefits from balanced parameter contribution in a case study of knowledge distillation. 
\subsection{Sensitivity Definition}
The sensitivity (also named importance scores) of a set of parameters represents the impact on the loss magnitude when the parameters are zeroed-out. It suggests that higher-sensitivity parameters contribute more to the loss. Consider a model paramterized as $\bm{\Theta}$. We denote the model loss as $\mathcal{L}(\bm{\Theta})$, gradient of the loss with respect to the model parameters as $\nabla_{\bm{\Theta}}\mathcal{L}(\bm{\Theta})$, the sensitivity of a set of parameters $\bm{\Theta}_s$ as $\mathcal{I}(\bm{\Theta}_s)$, model parameters with zeroed $\bm{\Theta}_s$ as $\bm{\Theta}_{-s}$. We evaluate sensitivity of $\bm{\Theta}_s$ by how much loss is preserved after zeroing $\bm{\Theta}_s$.
\begin{equation}
    \mathcal{I}(\bm{\Theta}_s) = |\mathcal{L}(\bm{\Theta}) - \mathcal{L}(\bm{\Theta}_{-s})|
    \label{eq:score-define1}
\end{equation}
The equation above implies that the larger the absolute loss change, the more sensitive $\bm{\Theta}_s$ is and the more contribution to the loss it makes. However, it is not practical to forward pass the model every time to compute the sensitivity of an arbitrary set of parameters. Thus, we utilize a first-order Taylor expansion of $\mathcal{L(\cdot)}$ with respect to $\bm{\Theta}_s$ at $\bm{\Theta}$ to approximate $\mathcal{I}(\bm{\Theta}_s)$.
\begin{equation}
    \mathcal{I}(\bm{\Theta}_s) \approx |\bm{\Theta}_s^T\nabla_{\bm{\Theta}}\mathcal{L}(\bm{\Theta})|
    \label{eq:score-define2}
\end{equation}

\begin{figure}[ht]
    \centering
    \resizebox{1\linewidth}{!}{
    \includegraphics[width=6.5cm]{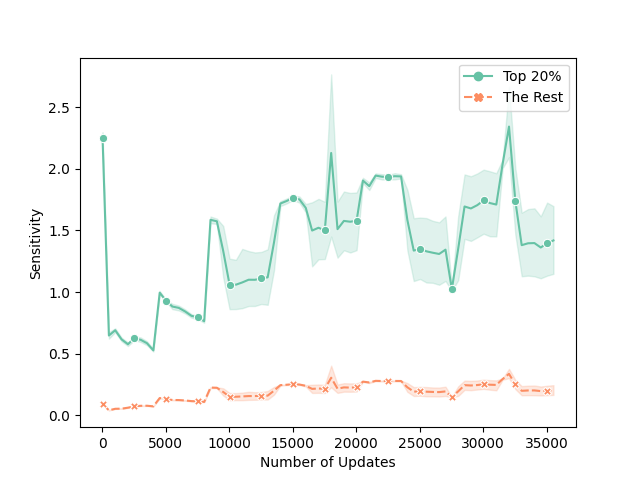}}
    \caption{The mean sensitivity along with the number of training updates of top 20\% most sensitive parameters and the rest of 80\% parameters. Though the top 20\% of parameters have significant variance in sensitivity at various stages of training, the bottom 80\% consistently have very little sensitivity.}
    \label{fig:sensitivity_vs_step}
\end{figure}

\subsection{Contribution Gap Among Parameters}
\label{sec:contri_gap}
Previous model pruning studies \citep{sanh2020movement,xiao2019autoprune} have shown that a small subset of parameters (e.g., 20\%) are extraordinarily effective for training, and the model performance is not significantly sacrificed after pruning. We attribute the success of pruning to the much larger contribution of high-sensitivity parameters over low-sensitivity parameters. We take machine translation on the IWSLT'14 German$\rightarrow$English (De$\rightarrow$En) dataset as our study object\footnote{Please find training details in Appendix \ref{app:train_detail}.}. We focus on the transformer$_{\texttt{small}}$ architecture \citep{vaswani2017attention}. We use Equation \ref{fig:sensitivity_vs_step} to track the sensitivity of each individual parameter and  visualize the mean sensitivity of the current top 20\% most sensitive parameters and the rest of 80\% parameters with the increasing of training updates in Figure \ref{fig:sensitivity_vs_step}. The sensitivity of the remaining 80\% of parameters are small and close to zero, but much larger for the top 20\% parameters\footnote{The best model is at 15K steps, where their gap is around 10 times apart.}.

\subsection{Benefits of Balanced Contribution}
\label{sec:kd}
We highlight the large contribution gap between parameters, and argue that the success of pruning is due to the modest contribution of low-sensitivity parameters. However, we take an alternative argument and pose the questions:  \textbf{ Do we overlook possible contributions of the low-sensitivity parameters when focusing on high sensitivity parameters? Will model performance improve when all parameters in a model contribute equally?} Here, we first define the degree of contribution balance and investigate a case study on knowledge distillation to show the benefits of more balanced contribution.
\paragraph{Degree of Contribution Balance} We define the degree of contribution balance to be simply evaluating the standard deviation of all parameter sensitivity. A lower standard deviation means that there is a more balanced contribution.

\begin{figure}[ht]
    \centering
    \resizebox{1\linewidth}{!}{
    \includegraphics[width=7.5cm]{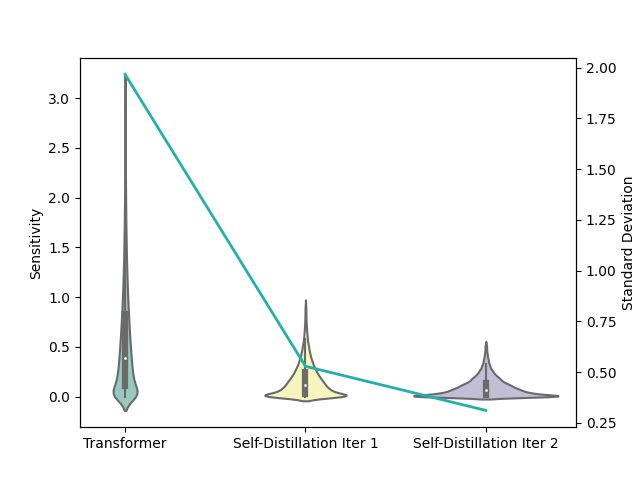}}
    \caption{Sensitivity distribution (violin plots aligned with left y-axis) along with their standard deviation (green curve aligned with right y-axis) in each iteration of self-distillation (the student and teacher use the same architecture).}
    \label{fig:kd_distribution}
\end{figure}

\begin{table}[ht]
\centering
\resizebox{0.7\linewidth}{!}{
\begin{tabular}{l|c}
\hline
Method      & De$\rightarrow$En \\ \hline
Transformer & 33.07               \\
Self-Distillation Round 1   & 34.87               \\
Self-Distillation Round 2   & 35.03               \\ \hline
\end{tabular}
}
\caption{SacreBLEU of iterative self-distillation on IWSLT'14 De$\rightarrow$En task.}
\label{tab:kd_result}
\end{table}
\paragraph{A Case Study on Knowledge Distillation}
We here take naive knowledge distillation (KD) \citep{hinton2015distilling} as a case study. We tie the success of KD to the more balanced contribution among parameters. KD aims to transfer knowledge from a teacher model to a student model. Specifically, the student model tries to minimize the Kullback–Leibler (KL) divergence between its output $p_s$ and the gold label $y$, and between $p_s$ and output of the teacher $p_t$. We here formulate a naive KD objective.
\begin{equation}
    \min \mathbbm{KL}(y\parallel p_s) + \mathbbm{KL}(p_t\parallel p_s)
    \label{eq:kd}
\end{equation}
Commonly, the teacher is a high-capacity model and student is more compact. However, recent studies \citep{furlanello2018born,zhang2019your,fang2020seed} show that the student model can significantly outperform the teacher when the student use the same architecture (and consequently, number of parameters) as the teacher, termed \textbf{self-distillation}. Using the previously described machine translation task in Section \ref{sec:contri_gap}, we conduct self-distillation experiments and iterate self-distillation twice, i.e., the student taught by the regular transformer model becomes a new teacher for the next student. In Table \ref{tab:kd_result}, we report \texttt{sacreBLEU} \citep{post-2018-call}. Similar to the previous literature, model performance substantially increase after each round of self-distillation. This surprising result is referred to in the literature as 'dark knowledge' \citep{gotmare2018closer,zhang2019your}. Some studies try to understand the `dark knowledge', e.g., in the view of regularization \citep{yuan2020revisiting} or ensemble \citep{allen2020towards}, but they only explain how it leads to performance improvements instead of how the model itself changes. Here, we argue that the `dark knowledge' transferred from teachers to such students is actually due to the more balanced contribution among parameters. We visualize the sensitivity distribution of all models via violin plots with their standard deviation in Figure \ref{fig:kd_distribution}\footnote{Implementation details of plotting are in Appendix \ref{app:sensitivity_vis}.}. Importantly, the parameter sensitivity becomes more balanced after each round of self-distillation. \textbf{We therefore argue that the effectiveness of self-distillation is caused by more balanced parameter contribution.}

Even though we hypothesize that balanced contributions explain why models improve under self-distilation, \textit{balanced contribution is not a sufficient condition for model improvement}. For instance, in an extreme case, all parameter values are 0, indicating that all parameters have equal contribution, but the model performance is nonsense. However, we hypothesize that the model generalization performance benefits from the constraints of contribution balance during training. Hence, this motivates us to propose a constraint term during training to improve the model generalization performance.

\section{Proposed Method}
\label{sec:method}
In the previous section, we showed two important findings; the large contribution gap between high- and low-sensitivity parameters, and that there is little understanding of the correlation between the better performance and more balanced contribution in self-distillation. In this section, we propose a general method to balance the parameter sensitivity (contribution) to improve the model performance.
\subsection{Intra-Distillation}
\label{sec:sd}
The sensitivity of parameters implies the degree of their contribution. We define our problem into minimizing the sensitivity difference among parameter groups. We randomly sample $K$ small groups of parameters\footnote{These groups of parameters may overlap.} $\{\bm{\Theta}_{s^1}, \cdots, \bm{\Theta}_{s^i}, \cdots, \bm{\Theta}_{s^K}\}$. Balancing sensitivity among all groups can be formulated as the following problem:
\begin{equation}
    \min  \sum_{i=1}^K\sum_{j=1}^K |\mathcal{I}(\bm{\Theta}_{s^i}) - \mathcal{I}(\bm{\Theta}_{s^j})|.
    \label{eq:min1}
\end{equation}
Based on Equation \ref{eq:score-define1}, it is equivalent to
\begin{multline}
 \min \sum_{i=1}^K\sum_{j=1}^K \bigl||\mathcal{L}(\bm{\Theta}) - \mathcal{L}(\bm{\Theta}_{-s^i})|- \\
    |\mathcal{L}(\bm{\Theta}) - \mathcal{L}(\bm{\Theta}_{-s^j})|\bigr|.
\end{multline}
Recall that $\bm{\Theta_{-s^i}}$ refers to the all parameters but zeroing out $\bm{\Theta_{s^i}}$. To facilitate training by not calculating $\mathcal{L}(\bm{\Theta})$ , we instead minimize the upper bound of the above objective\footnote{$||a|-|b||\leq |a-b|, \forall a,b\in\mathbbm{R}$.}. 

\begin{equation}
    \min \sum_{i=1}^K\sum_{j=1}^K |\mathcal{L}(\bm{\Theta}_{-s^i}) - \mathcal{L}(\bm{\Theta}_{-s^j})|
    \label{eq:min2}
\end{equation}
We denote the outputs of the model with $\bm{\Theta}_{-s^i}$ and $\bm{\Theta}_{-s^j}$ as $p_i$ and $p_j$, respectively. When we dissect Equation \ref{eq:min2} deeper, it actually tries to minimize the difference between $\mathcal{D}(y, p_i; \bm{\Theta}_{-s^i})$, the distance of the gold labels $y$ and $p_i$, and $\mathcal{D}(y, p_j; \bm{\Theta}_{-s^j})$, the distance of $y$ and $p_j$
\begin{equation}
    \min \sum_{i=1}^K\sum_{j=1}^K |\mathcal{D}(y, p_i; \bm{\Theta}_{-s^i}) - \mathcal{D}(y, p_j; \bm{\Theta}_{-s^j})|,
    \label{eq:min3}
\end{equation}
where $\mathcal{D}$ can be any similarity metrics, e.g., mean squared error (MSE) for regression tasks and Kullback–Leibler (KL) divergence for classification tasks. Instead of considering the loss difference between each pair of $p_i$ and $p_j$, we straightforwardly minimize the outputs $\{p_1, \cdots, p_i, \cdots, p_K\}$ without using $y$ as an intermediary. Most deep learning tasks can be categorized into classification and regression tasks. The outputs of classification tasks are probabilities while outputs of regression tasks could be any values. For classification tasks to which most NLP problems are boiled down, we propose a novel method, \textbf{X-divergence}, to measure the similarity among multiple distribution based on Jensen–Shannon (JS) divergence. We finalize our loss function to balance the parameter sensitivity as follows:

\begin{equation}
\begin{gathered}
    \mathcal{L}_{id} = \frac{1}{K}\sum_{i=1}^K \mathbbm{KL}(p_i\parallel\bar{p}) + \mathbbm{KL}(\bar{p}\parallel p_i) \\
    \text{where }\bar{p} = \frac{1}{K}\sum_{i=1}^Kp_i
    \label{eq:min_kl}.
\end{gathered}
\end{equation}
Here, we reduce the computation complexity from $\mathcal{O}(K^2)$ to $\mathcal{O}(K)$ compared to Equation \ref{eq:min3}. Different from the JS divergence that only calculates the KL divergence between $p_i$ and the `center' of all distributions $\bar{p}$, X-divergence also considers the KL divergence between $\bar{p}$ and $p_i$. We show that our X-divergence substantially outperforms JS divergence in Section \ref{sec:ex_mt}.

For regression tasks, we simply replace X-divergence with MSE to measure their similarity.
\begin{equation}
    \mathcal{L}_{id} = \frac{1}{K}\sum_{i=1}^K (p_i-\bar{p})^2
    \label{eq:min_mse}
\end{equation}

\subsection{Task-Agnostic Implementation}
\label{sec:implementation}
Intra-Distillation is easily  implemented into any deep learning task, without any model architecture modification. As presented in the previous Section \ref{sec:sd}, our final intra-distillation objective is to minimize the `distance' of $K$ outputs generated by sub-models which have $K$ different groups of disabled parameters. In the practical implementation, we run a forward-pass of the model $K$ times. For each pass, \textbf{we use dropout to simulate disabling a small subset of parameters}\footnote{Parameters connected to dropped nodes are equivalent to being disabled.}, and obtain the output. Thus, the final loss objective we want to optimize is composed of the regular training loss from each pass $\mathcal{L}(\bm{\Theta}_{-s^i})$ and intra-distillation loss $\mathcal{L}_{id}$.
\begin{equation}
    \min \frac{1}{K}\sum_{i=1}^K\mathcal{L}(\bm{\Theta}_{-s^i}) + \alpha \mathcal{L}_{id}
    \label{eq:final_loss}
\end{equation}
$\alpha$ is a hyper-parameter to control the strength of intra-distillation. The composition of the final loss is similar to the knowledge distillation loss in Equation \ref{eq:kd}. However, our second term minimizes the difference of outputs from the same model while  knowledge distillation minimizes the difference between the student and teacher (an external model).

\subsection{Adaptive Intra-Distillation}
\label{sec:adaptive-sd}
We notice that the intra-distillation term could slow down the convergence speed at the beginning of training, especially when it comes to a large $\alpha$ such as 5 or 10. More details will be discussed in Section \ref{sec:analysis-adpative}. Hence, we design an adaptive $\alpha$ algorithm that makes $\alpha$ small at the beginning of training and then becomes large afterwards to accelerate the convergence speed. \textbf{Ideally, $\alpha$ grows slowly at first and gets large quickly in the middle of training}. We denote $N$ as the total number of training steps and $x$ as current step. Our adaptive $\alpha'$ is formulated as follows.

\begin{gather}
\alpha'=
\begin{cases}
    \frac{p^\gamma\alpha}{N^\gamma}x^\gamma & x<\frac{N}{p}\\
    \alpha & x \geq \frac{N}{p}
\end{cases} \\
    \text{where } \gamma=\log_{\frac{p}{q}}\frac{1}{\alpha}
\label{eq:adaptive_alpha}
\end{gather}
An illustration of the growth of $\alpha'$ is shown in Figure \ref{fig:function}. Here, $p$ and $q$ are two sentinels ($q$ > $p$ > 0) to control the growth speed of $\alpha'$. Before the number of updates hits $\frac{N}{q}$, $\alpha'$ increase slowly from 0 to 1, because we want the model to pay less attention to intra-distillation. When training achieves $\frac{N}{q}$, $\alpha'$ is 1. $\mathcal{L}_{sd}$ now has the same weight as ordinary training loss. Then, the weight of intra-distillation should be raised substantially. $\alpha'$ increase quickly from 1 to $\alpha$ before update step achieves $\frac{N}{p}$. At the end, $\alpha'$ will be $\alpha$ constantly in the rest of the training steps.  Note that we only apply adaptive intra-distillation in the case of $\alpha > 1$. Otherwise, it is unnecessary to use adaptive learning.

\begin{figure}[ht]
    \centering
    \resizebox{1\linewidth}{!}{
    \includegraphics[width=7.5cm]{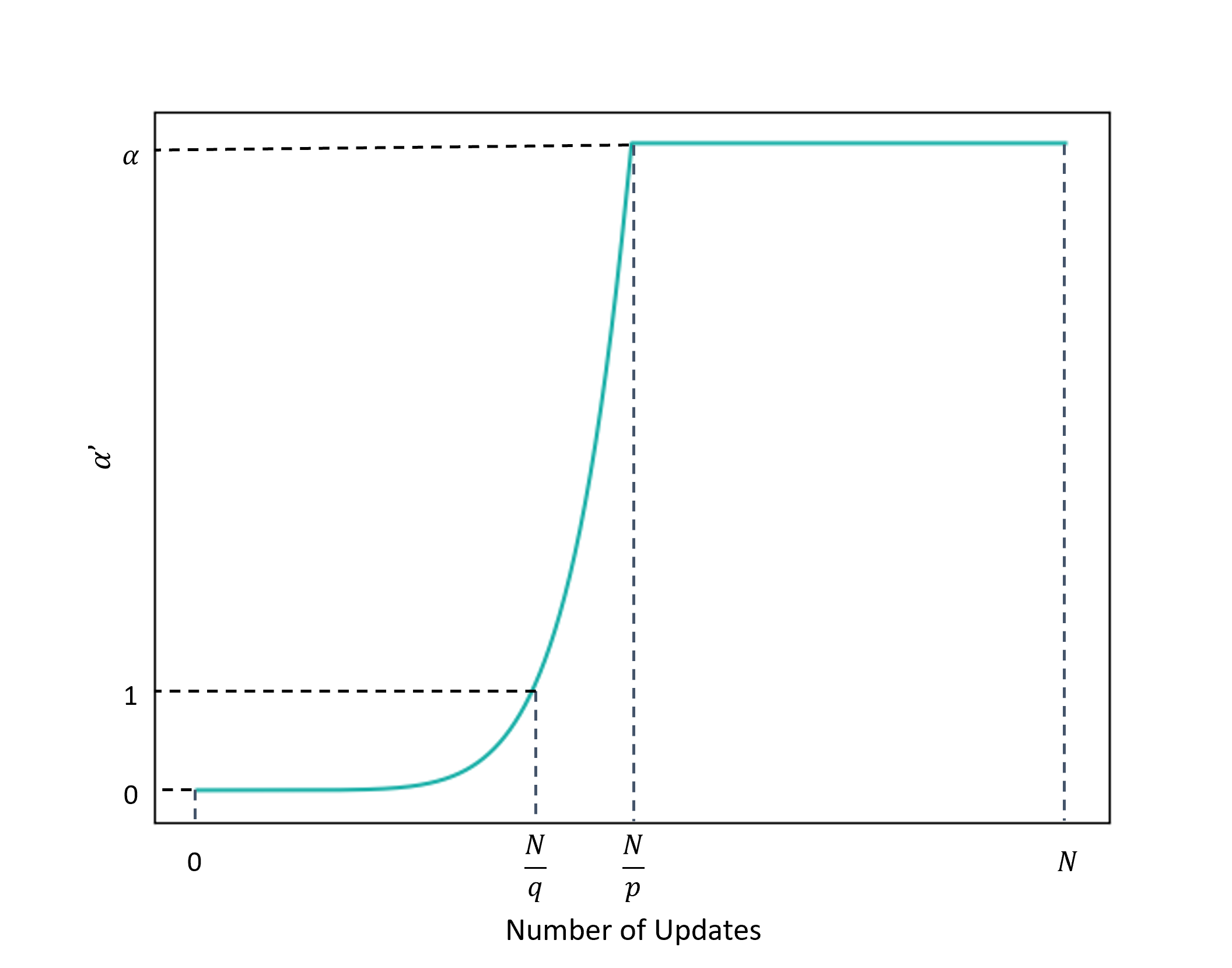}}
    \caption{An example of how $\alpha'$ increase with the control of $p$ and $q$. The relationship between $\frac{p}{q}$ and $\alpha$ determine shape (power) of the function. In a special case, adaptive learning is a linear increase if $q=\alpha p$.}
    \label{fig:function}
\end{figure}
$p$ and $q$ are two flexible hyper-parameters to control the weighting assigned to the intra-distillation during training. Note that a linear increase is a special case when $q=\alpha p$.

\begin{table*}[ht]
\centering
\resizebox{1\linewidth}{!}{
\begin{tabular}{l|ccccccccc|c}
\hline
Methods               & Ar             & De             & Es             & Fa             & He             & It             & Nl             & Pl             & Avg.           & \# parameters \\ \hline
Transformer \citep{vaswani2017attention}             & 28.62          & 33.07          & 38.89          & 19.64          & 35.13          & 32.76          & 35.90          & 19.69          & 30.46          & 49.98M        \\
SAGE \citep{liang2021no}                    & 29.76          & 33.51          & 39.22          & 20.01          & 36.19          & 32.77          & 36.12          & 20.45          & 31.00          & 49.98M        \\
R-Drop \citep{wu2021r}                  & 32.58          & 35.43          & 41.00          & 22.80          & 38.63          & 34.96          & 38.01          & 23.24          & 33.33          & 49.98M        \\
Switch Transformer \citep{fedus2021switch}       & 28.24          & 33.11          & 38.43          & 19.50          & 34.42          & 31.93          & 34.59          & 19.82          & 30.00          & 87.81M        \\ \hline
Intra-Distillation (JS Divergence, Ours) & 32.68 & 35.19 & 41.44 & 22.78 & 38.66 & 34.89 & 38.16 & 23.12 & 33.37 & 49.98M        \\
Intra-Distillation (X-Divergence, Ours) & \textbf{33.42} & \textbf{36.10} & \textbf{41.82} & \textbf{23.78} & \textbf{39.31} & \textbf{35.43} & \textbf{38.53} & \textbf{23.61} & \textbf{34.00} & 49.98M        \\ \hline
\end{tabular}
}
\caption{Results of 8 language pairs (Xx$\rightarrow$En) on the IWSLT'14 dataset.}
\label{tab:iwslt14_result}
\end{table*}

\begin{table}[ht]
\centering
\resizebox{1\linewidth}{!}{
\begin{tabular}{l|cc}
\hline
Methods                  & En$\rightarrow$De & \# parameters \\ \hline
Transformer \citep{vaswani2017attention}& 27.56               & 275M          \\ \hline
SAGE \citep{liang2021no}    & 27.76                  & 275M          \\
R-Drop \citep{wu2021r}         & 28.10                   & 275M          \\
Switch Transformer \citep{fedus2021switch}     & 27.80               & 577M          \\ \hline
Intra-Distillation (Ours) & \textbf{28.62}      & 275M          \\ \hline
\end{tabular}
}
\caption{Results of the WMT'17 En$\rightarrow$De task. }
\label{tab:wmt17_result}
\end{table}

\begin{table*}[ht]
\centering
\resizebox{1\linewidth}{!}{
\begin{tabular}{lccccccccc}
\hline
\multicolumn{1}{l|}{\multirow{2}{*}{Methods}}                                            & QNLI                           & MNLI-m/mm                           & CoLA                           & QQP                            & STS-B                          & RTE                            & SST-2                          & MRPC                           & Avg.                            \\
\multicolumn{1}{l|}{}                                                                    & Acc.                           & Acc.                                & Mcc.                           & F1                             & P/S Corr.                      & Acc.                           & Acc.                           & F1                             & Score                           \\ \hline
\multicolumn{1}{l|}{$\text{BERT}_{\text{base}}$ \citep{devlin2019bert}} & 90.6                           & 84.7/83.6                           & 54.0                           & 71.1                           & 86.6                           & 66.8                           & 93.4                           & 88.6                           & 79.93                           \\
\multicolumn{1}{l|}{SAGE \citep{liang2021no}}                           & 90.8                           & 84.9/83.8                           & 54.5                           & 71.3                           & 87.1                           & \textbf{69.8} & \textbf{94.1} & \textbf{89.7} & 80.67                           \\
\multicolumn{1}{l|}{R-Drop \citep{wu2021r}}                             & 91.2                           & 85.0/84.3                           & 54.0                           & 72.3                           & 87.1                           & 66.6                           & 93.8                           & 88.1                           & 80.27                           \\ \hline
\multicolumn{1}{l|}{Intra-Distillation (Ours)}                                            & \textbf{91.7} & \textbf{85.2/84.2} & \textbf{55.1} & \textbf{72.4} & \textbf{87.5} & 67.5                           & \textbf{94.1} & 89.0                           & \textbf{80.74} \\ \hline
\end{tabular}
}
\caption{Results of the GLUE benchmark.}
\label{tab:gleu_result}
\end{table*}

\section{Experiments}
\label{sec:experiments}
We evaluate our method on widely used benchmarks for machine translation, natural language understanding and zero-shot cross-lingual transfer. We pass the model $K=3$ times for all experiments. We explain the influence of $K$ in Section \ref{sec:model_pass}. Note that we briefly describe key training settings for each task but leave details in Appendix \ref{app:train_detail}.
\subsection{Baselines}
We consider three baselines in our experiments. All baseline results are from our implementation and followed by the settings from the original papers.
\paragraph{SAGE} SAGE \citep{liang2021no} is a sensitivity-guided adaptive learning rate method, which encourages all parameters to be trained sufficiently by assigning higher learning rates to low-sensitivity parameters, and vice versa. SAGE is on the same study line of salience of redundant parameters as ours but using different methods.
\paragraph{R-Drop} R-drop \citep{wu2021r} is a recently proposed state-of-the-art method that focuses on minimizing the inconsistency between training and inference, rather than focusing on parameter sensitivity. However though motivated differently, this method derives a similar loss objective to our proposed intra-distillation. They pass the model twice and minimize the difference of two outputs by using the Jeffrey divergence (the term for the symmetric KL). However, the advantage of X-divergence is that it is bounded while Jeffrey divergence is not, which makes training more stable. We show that our proposed X-divergence for multi-pass learning with adaptive $\alpha'$ can achieve superior performance. Interestingly, we theoretically prove that Jeffrey divergence is the upper bound of  X-divergence in Appendix \ref{app:upper_bound}. 
\paragraph{Switch Transformer} Scaling up the number of parameters has been usually used for improving model performance. To show the parameter efficiency of our method, We also compare our method to a well-known sparsely activated model, switch transformer \citep{fedus2021switch}, in machine translation tasks. Considering the huge memory expense, we here only consider 4-expert switch transformer.
\subsection{Machine Translation}
\label{sec:ex_mt}
\paragraph{Data and Settings} We consider both low- and high- resource data conditions. For the low-resource scenario, we collect 8 English-centric language pairs from IWSLT'14 (Xx$\rightarrow$En), including Arabic (Ar), German (De), Spanish (Es), Farsi (Fa), Hebrew (He), Italian (It), Dutch (Nl), Polish (Pl). The training pairs ranges from 89K to 160K. We use the transformer$_{\texttt{small}}$ architecture \citep{vaswani2017attention}. We set $\alpha=5,p=5, q=10, N=50\text{K}$ for adaptive intra-distillation. For the high-resource scenario, we consider WMT 17 En$\rightarrow$De translation task, whose corpus size is 4.5M. Following \citet{ott-etal-2019-fairseq}, we separate 40K training pairs as the validation set and \texttt{newstest2014} as the test set. We use the transformer$_{\texttt{big}}$ model and set $\alpha=5,p=6.25, q=10, N=50\text{K}$. For both scenarios, the dropout rate is 0.1 for attention layers and 0.3 for FFN layers. We tokenize all sentences by \texttt{sentencepiece} \citep{kudo-richardson-2018-sentencepiece}. We report \texttt{sacreBLEU} points \citep{post-2018-call}.

\paragraph{Results} Results for IWSLT'14 are show in Table \ref{tab:iwslt14_result}. SAGE outperforms the transformer baseline by 0.54 BLEU points on average, which matches the similar improvement in \citet{liang2021no}. Interestingly, the switch transformer is not parameter-efficient when we double the parameters. At best, it only provides modest improvements in some experiments and even degenerates the performance in others. R-drop is the most competitive method. However, we still achieve the best performance by boosting the transformer model \textbf{3.54} BLEU points on average. Moreover, Our X-divergence outperform JS divergence by 0.63 on average. In Table \ref{tab:wmt17_result}, similar observations also holds for the WMT'17 task, where we achieve the highest improvement (1.06 BLEU).

\subsection{Natural Language Understanding}
\paragraph{Data and Settings} We evaluate our methods and baselines on the General Language
Understanding Evaluation (GLUE) benchmark\footnote{We use Equation \ref{eq:min_mse} for STS-B becasue it is a regression task, while the others we still use Equation \ref{eq:min_kl}. } \citep{wang2018glue}. We fine-tune pre-trained BERT \citep{devlin2019bert} base model on each task of GLUE. We follow the hyperparameter settings suggested by \citet{liu2020microsoft}. To have a fair comparison to SAGE, we adopt Adamax \citep{kingma2014adam} optimizer. The dropout rate is 0.1. $\alpha$ for each task is in the range of \{0.5, 1.0, 1.5\}. Recall that we do not apply adaptive $\alpha$ to intra-distillation if $\alpha \leq 1$.
\paragraph{Results} We report the result of GLUE test set in Table \ref{tab:gleu_result}. Scores are calculated by GLUE online evaluation server. SAGE and our method achieve similar gains (0.74 vs. 0.79) over the BERT baseline on average. However, our method performs much better on large datasets, e.g., QNLI (105K) with a gain of 1.1, QQP (364K) with a gain of 1.3, while SAGE achieves modest improvements on these tasks. On the other hand, interestingly, SAGE is more effective on small datasets, e.g., RTE (2.4K) with a gain of 3.0 and MRPC (3.7K) with a gain of 1.1. Our method also outperform R-Drop by 0.47 on average.

\begin{table}[ht]
\centering
\resizebox{1\linewidth}{!}{
\begin{tabular}{l|cc}
\hline
Method                    & NER           & TyDiQA        \\ \hline
XLM-R                     & 64.6          & 55.8          \\
XLM-R + Intra-Distillation & \textbf{66.0} & \textbf{58.0} \\ \hline
\end{tabular}
}
\caption{Zero-shot cross-lingual transfer results (F1 score) on NER and TyDiQA. The scores are averaged by all languages and 5 trials with different random seeds.}
\label{tab:xtreme_result}
\end{table}

\subsection{Cross-Lingual Transfer}
\paragraph{Data and Settings} We consider a low-level and a high-level task for zero-shot cross-lingual transfer, i.e., Wikiann Named-Entity Recognition (NER) task \citep{ner} and Typologically Diverse Question Answering-Gold Passage (TyDiQA) \citep{artetxe2020cross}. We download datasets from the XTREME-R benchmark \citep{ruder-etal-2021-xtreme}. NER and TyDiQA cover 48 and 9 languages, respectively. Following \citet{xu-murray-2022-por}, the model architecture of NER is based on pre-trained XLM-R$_\text{large}$ attached with a feed-forward token-level classifier. For TydiQA, the representations of all subwords in XLM-R$_\texttt{base}$ are input to a span classification head –-- a linear layer computing the start and the end of the answer. The models are only trained on English and then evaluated on all languages. We set dropout rate as 0.1 and run 10 and 15 epochs for NER and TyDiQA, both with $\alpha=1$.
\paragraph{Results} Averaged results (F1 scores) among all languages are shown in Table \ref{tab:xtreme_result}. We run the model 5 times with 5 different random seeds and report the averaged F1 score. The models have better overall performance after applying intra-distillation. NER achieves 1.4 F1 improvement on average. The high-lever task, TyDiQA,  benefits more from intra-distillation, and obtain 2.2 F1 improvement. Please find full results on all languages in Appendix \ref{app:full_cs_results}.
\begin{figure}[ht]
    \centering
    \resizebox{1\linewidth}{!}{
    \includegraphics[width=7.5cm]{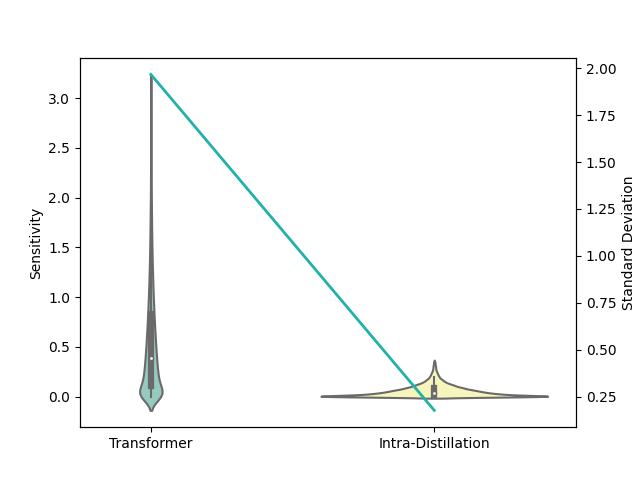}}
    \caption{Sensitivity distribution comparison along with their standard deviation between the models with and without using intra-distillation on the IWSLT'14 De$\rightarrow$En MT task.}
    \label{fig:dis_compare}
\end{figure}


\begin{figure}[ht]
     \centering
     \begin{subfigure}[b]{0.48\textwidth}
         \centering
         \resizebox{1\linewidth}{!}{
         \includegraphics[width=\textwidth]{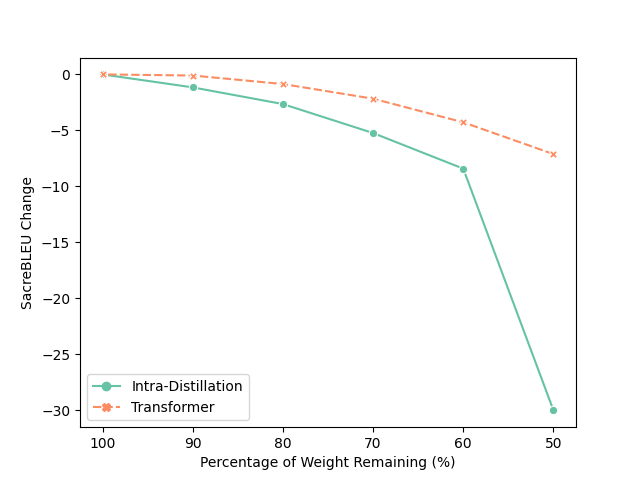}}
         \caption{Change in BLEU scores}
         \label{fig:prune_change}
     \end{subfigure}
     \hfill
     \begin{subfigure}[b]{0.48\textwidth}
         \centering
         \resizebox{1\linewidth}{!}{
         \includegraphics[width=\textwidth]{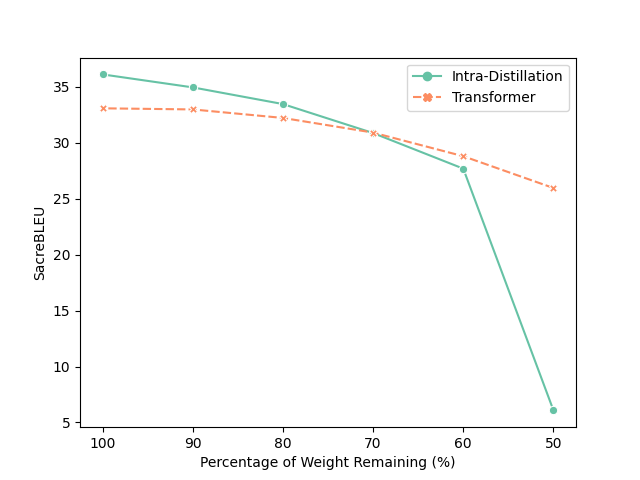}}
         \caption{BLEU scores}
         \label{fig:prune_value}
     \end{subfigure}
     \caption{Model performance with different pruning ratio for the translation task.} 
     \label{fig:prune}
\end{figure}

\section{Analysis}
\subsection{More Balanced Contribution}
\label{sec:balanced_contribution}
We here focus on analyzing IWSLT'14 De$\rightarrow$En translation task, but we also show the similar findings on the QQP task in Appendix \ref{app:qqp}. We show the sensitivity distribution comparison with and without intra-distillation in Figure \ref{fig:dis_compare}. The sensitivity is computed on the model which performs best on the valid set. After intra-distillation, sensitivity distribution is more concentrated, implying parameter contribution is more balanced as our goal.

We are also interested in the contribution of parameters with respect to the downstream metric. We compare a typically trained transformer model to one trained with intra-disillation, pruning up to 50\% of the parameters. We prune parameters in order of sensitivity, starting with the least sensitive parameters.  As shown in Figure~\ref{fig:prune}, BLEU drops significantly faster for the intra-distillation model, as more parameters are pruned. This suggests that the low-sensitivity parameters of the intra-distilled model contribute much more (to task performance) than in the regular transformer model, so the model generalization degenerates faster without them. Particularly, we observe that intra-distillation significantly improves the contribution of the parameters  within the lowest 40\%-50\% parameter range. After removing them, BLEU further drops around 20 points (yielding a BLEU score near 5, which is basically an unintelligible  translation), but the regular transformer only drops less than 3 points and still scores over 25 BLEU in total. Thus, intra-distillation shows the importance of these lower-sensitivity parameters and the significant performance degeneration after pruning them. 


\begin{figure}[ht]
    \centering
    \resizebox{1\linewidth}{!}{
    \includegraphics[width=7.5cm]{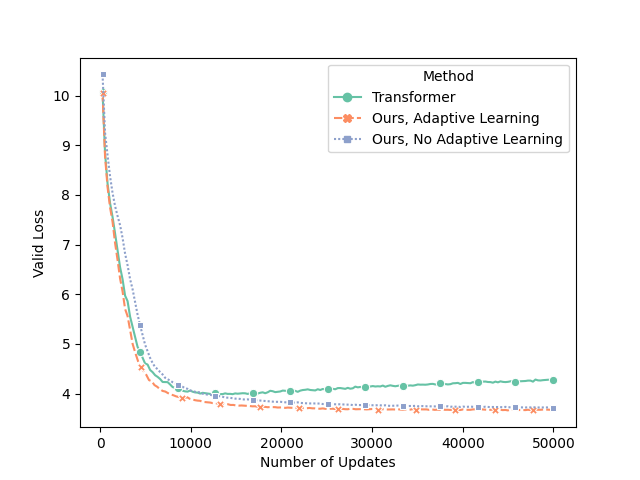}}
    \caption{Valid loss curves comparison between intra-distillation with and without adaptive learning. We also add the loss curve of regular transformer as a reference.}
    \label{fig:update_vs_loss}
\end{figure}

\subsection{Adaptive Learning for Intra-Distillation}
\label{sec:analysis-adpative}
We here show how the adaptive learning method helps convergence. We conduct an apples-to-apples comparison in IWSLT'14 De$\rightarrow$En translation task between intra-distillation with and without dynamic $\alpha'$. Our $\alpha$ is 5 as set above. As shown in Figure \ref{fig:update_vs_loss}, without adaptive learning, the valid loss is substantially higher than the loss of the regular transformer at the beginning of training. However, adaptive learning eliminates this issue and the loss even drops faster than the baseline model. Moreover, the valid loss with adaptive learning is always lower than the one without adaptive learning at the same training step.

\subsection{Number of Model Passes}
\label{sec:model_pass}
We examine the impact of the number of model passes $K$. We conduct experiments for IWSLT'14 De$\rightarrow$En with various $K$ ranging from 2 to 6. Figure \ref{fig:step_vs_sacrebleu} shows that multi-pass training is crucial to the model performance. The resultant gain obtained is 0.4 when we increase from 2 passes to 3 passes. However, the performance is similar if the number of passes is larger than 3. Although the 4-pass model works slightly better than 3-pass model, we still pass the model 3 times for all tasks considering the computational cost and slight improvement.
\begin{figure}[ht]
    \centering
    \resizebox{1\linewidth}{!}{
    \includegraphics[width=7.5cm]{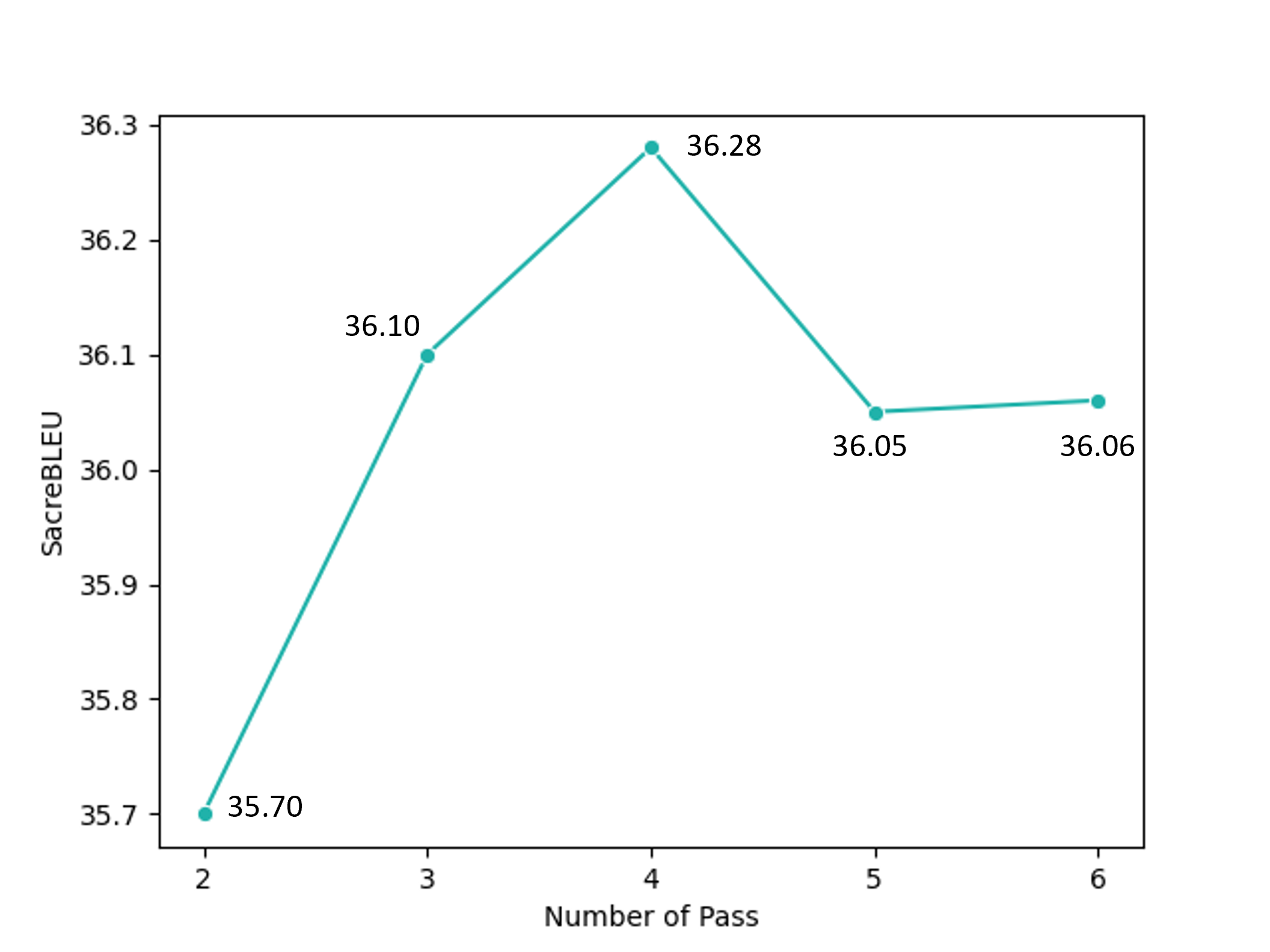}}
    \caption{IWSLT'14 De$\rightarrow$En BLEU vs. number of model passes $K$.}
    \label{fig:step_vs_sacrebleu}
\end{figure}

\section{Conclusions}
Taking an opposite view from pruning redundant parameters, we questioned whether we overlook the potential of these redundant parameters and encouraged all parameters to contribute equally to improve the model generalization performance. We first introduced the concept of \textit{degree of contribution balance} to describe how balanced of all parameters is. Then, we used balanced parameter contribution to explain the `dark knowledge' that successfully transfers in knowledge distillation, by analyzing the contribution gap among parameters within a model. With the goal of adding a constraint term to balance the parameter contribution, we proposed \textbf{intra-distillation} with a strength-adaptive learning method. With wide-ranging experiments and analysis on machine translation, natural language understanding and zero-shot cross-lingual transfer tasks, we demonstrated that intra-distillation are capable of improving the model performance significantly and  balance the parameter contribution effectively.

\section*{Limitations}
This method modifies the training objective and is not specific to data. As such, any standard limitation of a neural training method will apply here, such as biases in data, hyperparameter choices, etc. However, being data agnostic also implies that the method should in theory be language and task agnostic. We've shown improvements on multiple languages from diverse language families on multiple tasks, yet naturally, this list is non-exhaustive and limited to the NLP domain. We expect the method generalizes to tasks outside of language, but have not explored these. Furthermore, since we need to pass the model $K$ times, the method incurs a higher time cost  (or more memory cost if we concatenate the same $K$ inputs for one pass computation\footnote{One can concatenate the same $K$ inputs and feed them to the model once instead of passing the model $K$ times, which cost more memory rather than training time.}). However, though this cost is a limitation, we argue that it is acceptable given a user is cognizant of this and compares it to the improved performance.  

\section*{Acknowledgements}
We thank Marc Marone, Xuan Zhang and Shuoyang Ding and anonymous reviewers for their helpful suggestions. This work was supported in part by IARPA BETTER (\#2019-19051600005). The views and conclusions contained in this work are those of the authors and should not be interpreted as necessarily representing the official policies, either expressed or implied, or endorsements of ODNI, IARPA, or the U.S. Government. The U.S. Government is authorized to reproduce and distribute reprints for governmental purposes notwithstanding any copyright annotation therein.

\bibliography{anthology,custom}
\bibliographystyle{acl_natbib}

\clearpage
\appendix
\section{Training Details}
\label{app:train_detail}
\subsection{IWSLT'14 Translation}
We use the same training configuration for all 8 language pairs. We filter out the training pairs whose length ratio is larger than 1.5 or one of length is longer than 175 tokens. We use small transformer architecture \citep{vaswani2017attention},  with FFN dimension size 1024, attention dimension size 512 and 4 attention heads. The batch size is 4096 tokens. We jointly train a 12K bilingual vocabulary by using sentencepiece \citep{kudo-richardson-2018-sentencepiece} for each language pair. The maximum learning rate is 0.0005. The optimizer is Adam \citep{kingma2014adam} with \texttt{inverse\_sqrt} learning rate scheduler and weight decay of 0.0001. The maximum training update is 50K with 8K warm-up steps. At inference time, we use beam search with width 5 and use a length penalty of 1.
\subsection{WMT'17 Translation}
We filter out the training pairs whose length ratio is larger than 1.5 or one of length is longer than 256 tokens. We use large transformer architecture \citep{vaswani2017attention},  with FFN dimension size 4096, attention dimension size 1024 and 16 attention heads. The batch size is 4096 tokens but we accumulate gradients for 16 times. We jointly train a 32K bilingual vocabulary by using sentencepiece \citep{kudo-richardson-2018-sentencepiece}. The maximum learning rate is 0.0005. The optimizer is Adam \citep{kingma2014adam} with \texttt{inverse\_sqrt} learning rate scheduler and weight decay of 0.0001. The maximum training update is 50K with 4K warm-up steps. At inference time, we use beam search with width 4 and use a length penalty of 0.6.
\subsection{GLUE benchmark}
We use pre-trained BERT$_{\texttt{base}}$ model \citep{devlin2019bert} and fine-tune it on each GLUE task. We set the maximum sentence length as 128. Batch size is 32 sentences. The optimizer is Adamax \citep{kingma2014adam} with 2e-4 learning rate. We run 20 epochs for each task. The result for STS-B is the Pearson correlation. Matthew’s correlation is used for CoLA. F1 is used for QQP and MRPC. Other tasks are measured by Accuracy. We leave the detailed settings of $\alpha$, $p$ and $q$ of every GLUE task in Table \ref{tab:gleu_setting}.

\begin{table}[ht]
\centering
\resizebox{1\linewidth}{!}{
\begin{tabular}{l|cccccccc}
\hline
      & QNLI & MNLI & CoLA & QQP & STS-B & RTE & SST-2 & MRPC \\ \hline
$\alpha$ & 0.5  & 0.5  & 0.5  & 0.5 & 1     & 1   & 0.5   & 1.5  \\
$p$     & -    & -    & -    & -   & -     & -   & -     & 6.67 \\
$q$     & -    & -    & -    & -   & -     & -   & -     & 10   \\ \hline
\end{tabular}
}
\caption{Settings of $\alpha$, $p$ and $q$ of every task at GLUE benchmark.}
\label{tab:gleu_setting}
\end{table}

\subsection{Xtreme-R Benchmark}
We consider NER and TyDiQA tasks to evaluate the effectiveness of intra-distillation in zero-shot cross-lingual transfer learning. NER and TyDiQA respectively contains 48 and 9 languages. For NER, we use XLM-R$_\texttt{large}$ model architecture \citep{conneau2020unsupervised}. The max length is 128. We train the model for 10 epochs with learning rate 2e-5, batch size 8 and gradient accumulation 4. For TyDiQA, we use XLM-R$_\texttt{base}$ model architecture. The max length is 384. We train the model for 15 epochs with learning rate 3e-5, batch size 8 and gradient accumulation 4.

\section{Sensitivity Distribution Visualization Details}
\label{app:sensitivity_vis}
Sensitivity of each parameter approximates to the absolute multiplication of its value and gradient. We randomly pick 100 batches and feed to the model to retrieve the gradients. Note that We also remove the top 1\% highest-sensitive parameters to ease the illustration. We store the sensitivity of each parameter and randomly sample 10\% of them to visualize them via violin plots.

\section{The Bound of X-Divergence}
\label{app:upper_bound}
Here, we show that our X-divergence is upper bounded by the Jeffrey divergence. Usually, Jeffrey divergence only serves for two distributions:
\begin{equation}
    J(p_1,p_2) = \mathbbm{KL}(p_1\parallel p_2) + \mathbbm{KL}(p_2\parallel p_1)
\end{equation}
We generalize it to measure multiple (say, $K$) distributions:
\begin{multline}
    J(p_1,\cdots,p_K) = \sum_{i=1}^{K}\sum_{j=1}^{K}\mathbbm{KL}(p_i\parallel p_j) \\+ \mathbbm{KL}(p_j\parallel p_i)
\end{multline}
Our proposed X-divergence is formulated as follows:
\begin{multline}
    X(p_1,\cdots,p_K) = \frac{1}{K}\sum_{i=1}^K \mathbbm{KL}(p_i\parallel\bar{p}) \\+ \mathbbm{KL}(\bar{p}\parallel p_i) 
\end{multline}

\begin{theorem}
X-divergence is upper bounded by the Jeffrey Divergence:
\begin{equation}
    X(p_1,\cdots,p_K) \leq \frac{1}{K^2} J(p_1,\cdots,p_K)
    \label{eq:app:0}
\end{equation}
\end{theorem}

\begin{proof}
We separate the proof in two parts. We prove that the first and second term of $J$ divergence are the upper bound of the first and second term of X-divergence, respectively, i.e., 
\begin{equation}
    \frac{1}{K}\sum_{i=1}^K \mathbbm{KL}(p_i\parallel\bar{p}) \leq \frac{1}{K^2}\sum_{i=1}^{K}\sum_{j=1}^{K}\mathbbm{KL}(p_i\parallel p_j)
    \label{eq:app:1}
\end{equation}
and
\begin{equation}
    \frac{1}{K}\sum_{i=1}^K\mathbbm{KL}(\bar{p}\parallel p_i) \leq \frac{1}{K^2}\sum_{i=1}^{K}\sum_{j=1}^{K}\mathbbm{KL}(p_j\parallel p_i).
    \label{eq:app:2}
\end{equation}
We first prove the Equation \ref{eq:app:1}. Since each $p_i \geq 0$, by the inequality of the arithmetic and geometric means, we have
\begin{equation*}
    \frac{\sum_{i=1}^Kp_i}{K} \geq \sqrt[K]{\prod_{i=1}^Kp_i}. \
\end{equation*}
Thus, it follows
\begin{equation*}
    \begin{split}
        &\frac{1}{K}\sum_{i=1}^K \mathbbm{KL}(p_i\parallel\bar{p}) \\&= \frac{1}{K}\sum_{i=1}^Kp_i\log(\frac{p_i}{\frac{\sum_{j=1}^Kp_j}{K}}) \\
        & \leq  \frac{1}{K}\sum_{i=1}^Kp_i\log(\frac{p_i}{\sqrt[K]{\prod_{j=1}^Kp_j}})\\
        & =\frac{1}{K}\sum_{i=1}^Kp_i(\log p_i - \frac{\sum_{j=1}^K\log p_j}{K})\\
        & = \frac{1}{K}\sum_{i=1}^Kp_i\frac{\sum_{j=1}^K\log p_i-\log p_j}{K} \\
        & = \frac{1}{K^2}\sum_{i=1}^K\sum_{j=1}^Kp_ilog\frac{p_i}{p_j} \\
        & = \frac{1}{K^2}\sum_{i=1}^{K}\sum_{j=1}^{K}\mathbbm{KL}(p_i\parallel p_j).
    \end{split}
\end{equation*}
Now, we have proved the Equation \ref{eq:app:1}, and move to the proof of Equation \ref{eq:app:2}. Consider that the function $f(x)=x\log x$ is a convex function. Based on the Jensen's inequality, we have
\begin{equation*}
    \frac{\sum_{i=1}^Kp_i}{K}\log(\frac{\sum_{i=1}^Kp_i}{K}) \leq \frac{1}{K}(\sum_{i=1}^Kp_i\log p_i).
\end{equation*}
Thus, it follows
\begin{equation*}
\begin{split}
    &\frac{1}{K}\sum_{i=1}^K\mathbbm{KL}(\bar{p}\parallel p_i)\\
    &=\frac{1}{K}\sum_{i=1}^K\frac{\sum_{j=1}^Kp_j}{K}(\log\frac{\sum_{j=1}^Kp_j}{K}-\log p_i)\\
    &\leq\frac{1}{K}\sum_{i=1}^K\frac{\sum_{j=1}^Kp_j\log p_j}{K} - \frac{\sum_{j=1}^Kp_j\log p_i}{K}\\
    & = \frac{1}{K^2}\sum_{i=1}^K\sum_{j=1}^Kp_j\log\frac{p_j}{p_i}\\
    & =  \frac{1}{K^2}\sum_{i=1}^K\sum_{j=1}^K\mathbbm{KL}(p_j\parallel p_i).
\end{split}
\end{equation*}
We also have proved the Equation \ref{eq:app:2}. Thus, the proof of Equation \ref{eq:app:0} is done.

\end{proof}

\section{Full Results of Zero-Shot Cross-Lingual Transfer}
\label{app:full_cs_results}
We leave the full results of zero-shot cross-lingual transfer learning on the NER, TyDiQA task in Table \ref{app:tab:full_ner} and Table \ref{app:tab:full_tydiqa}, respectively.
\begin{table*}[ht]
\centering
\resizebox{1\linewidth}{!}{
\begin{tabular}{lccccccccccccccccccccccccc}
\hline
Methods                    & ar             & he             & vi             & id             & jv             & ms             & tl             & eu             & ml             & ta             & te             & af             & nl             & en             & de             & el             & bn             & hi             & mr             & ur             & fa             & fr             & it             & pt             & es             \\ \hline
XLM-R                      & 45.75          & \textbf{55.35} & \textbf{78.67} & 52.47          & 61.35          & 69.65          & 71.95          & 56.37          & \textbf{65.79} & 55.82          & \textbf{52.85} & 78.34          & 83.76          & 84.50          & \textbf{78.78} & 78.38          & 74.39          & 69.71          & 61.87          & 54.85          & \textbf{56.82} & 79.78          & 81.39          & 81.91          & \textbf{76.64} \\
XLM-R + Intra-Distillation & \textbf{49.26} & 53.31          & 77.65          & \textbf{53.15} & \textbf{63.89} & \textbf{71.13} & \textbf{75.70} & \textbf{62.91} & 62.97          & \textbf{59.66} & 51.53          & \textbf{79.30} & \textbf{84.39} & \textbf{85.40} & 78.59          & \textbf{80.97} & \textbf{76.82} & \textbf{71.54} & \textbf{63.46} & \textbf{56.05} & 51.28          & \textbf{80.68} & \textbf{81.78} & \textbf{82.46} & 76.61          \\ \hline
                           & bg             & ru             & ja             & ka             & ko             & th             & sw             & yo             & my             & zh             & kk             & tr             & et             & fi             & hu             & qu             & pl             & uk             & az             & lt             & pa             & gu             & ro             & Avg.           &                \\ \hline
XLM-R                      & 81.32          & 70.60          & 18.31          & 66.37          & \textbf{57.28} & 1.02           & 69.86          & 32.90          & 51.97          & 27.06          & 50.46          & \textbf{79.30} & 77.79          & 79.65          & 80.13          & 54.62          & 80.89          & 74.48          & 67.61          & 76.87          & 48.62          & 61.59          & 82.98          & 64.56          &                \\
XLM-R + Intra-Distillation & \textbf{82.46} & \textbf{70.87} & \textbf{21.67} & \textbf{67.46} & 56.04          & \textbf{1.76}  & \textbf{70.09} & \textbf{39.77} & \textbf{57.62} & \textbf{29.19} & \textbf{53.84} & 78.48          & \textbf{78.65} & \textbf{80.87} & \textbf{81.33} & \textbf{56.93} & \textbf{82.38} & \textbf{80.65} & \textbf{70.47} & \textbf{78.90} & \textbf{51.43} & \textbf{64.49} & \textbf{90.90} & \textbf{65.97} &                \\ \hline
\end{tabular}

}
\caption{Full results (F1) of the zero-shot cross-lingual NER task.}
\label{app:tab:full_ner}
\end{table*}

\begin{table*}[ht]
\centering
\resizebox{1\linewidth}{!}{
\begin{tabular}{lcccccccccc}
\hline
Methods                    & ar             & bn             & fi             & id             & ko             & ru             & sw             & te             & en             & Avg.           \\ \hline
XLM-R                      & 62.53          & \textbf{42.24} & 61.82          & 70.62          & 42.99          & 57.75          & 56.40          & 43.23          & 65.51          & 55.80          \\
XLM-R + Intra-Distillation & \textbf{67.03} & 41.24          & \textbf{62.54} & \textbf{74.91} & \textbf{44.50} & \textbf{59.46} & \textbf{57.07} & \textbf{47.29} & \textbf{68.20} & \textbf{58.03} \\ \hline
\end{tabular}
}
\caption{Full results (F1) of the zero-shot cross-lingual TyDiQA task.}
\label{app:tab:full_tydiqa}
\end{table*}

\section{More Balanced Contribution in The QQP Task}
\label{app:qqp}
Similar to the findings in Section \ref{sec:balanced_contribution}, sensitivity of all parameters becomes more balanced after using intra-distillation (Figure \ref{app:fig:dis_compare_qqp}). Moreover, in the one-shot unstructured pruning, performance of the model which is trained with intra-distillation drops faster than the regular model (Figure \ref{app:fig:weight_f1}). This also implies that lower-sensitivity parameters contribute more than the regular ones.
\begin{figure}[ht]
    \centering
    \resizebox{1\linewidth}{!}{
    \includegraphics[width=7.5cm]{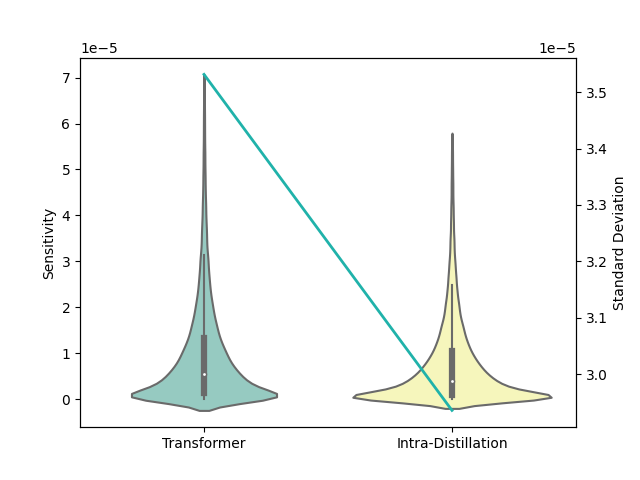}}
    \caption{Sensitivity distribution comparison along with their standard deviation between the models with and without using intra-distillation on the QQP task.}
    \label{app:fig:dis_compare_qqp}
\end{figure}


\begin{figure}[ht]
     \centering
     \begin{subfigure}[b]{0.45\textwidth}
         \centering
         \resizebox{1\linewidth}{!}{
         \includegraphics[width=\textwidth]{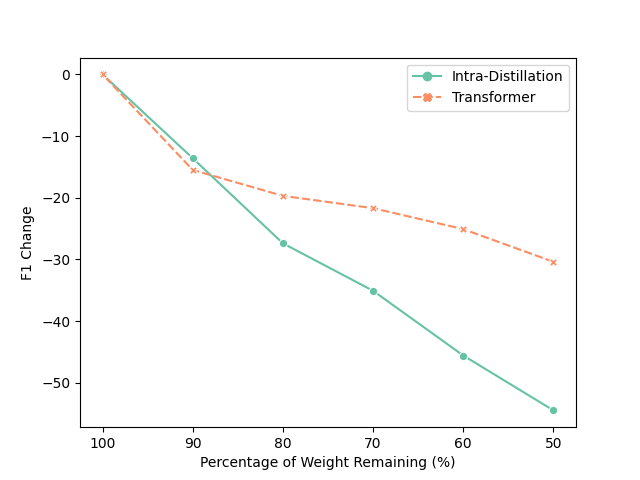}}
         \caption{Change in F1 scores}
         \label{fig:mt_dis_compare2}
     \end{subfigure}
     \hfill
     \begin{subfigure}[b]{0.45\textwidth}
         \centering
         \resizebox{1\linewidth}{!}{
         \includegraphics[width=\textwidth]{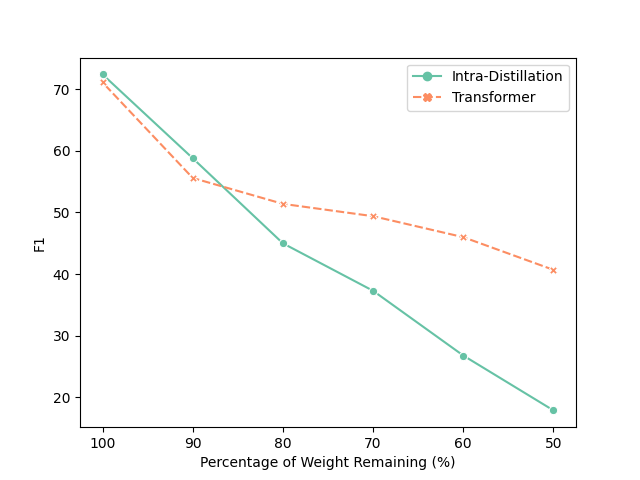}}
         \caption{F1 scores}
         \label{fig:qqp_dis_compare2}
     \end{subfigure}
     \caption{Model performance with different pruning ratio for the QQP task.} 
     \label{app:fig:weight_f1}
\end{figure}


\end{document}